\documentclass{article}

\usepackage{microtype}
\usepackage{graphicx}
\usepackage{subfigure}
\usepackage{booktabs} 
\usepackage{multirow}

\usepackage{hyperref}

\usepackage{amsmath}
\usepackage{amsthm}

\usepackage{stfloats}

\usepackage[accepted]{icml2019}

\icmltitlerunning{MBWPNM for Image Denoising}

\begin{document}

\twocolumn[
\icmltitle{Multi-band Weighted $ l_{p} $ Norm Minimization for Image Denoising}




\begin{icmlauthorlist}
\icmlauthor{Yanchi Su}{to}
\icmlauthor{Zhanshan Li}{goo}
\icmlauthor{Haihong Yu}{too}
\icmlauthor{Zeyu Wang}{tooo}\\
\text{suyanchi@gmail.com, lizs@jlu.edu.cn, yuhh@jlu.edu.cn, zyw17@mails.jlu.edu.cn}
\end{icmlauthorlist}

%
%

\icmlaffiliation{to}{Department of Computation, University of Torontoland, Torontoland, Canada}
\icmlaffiliation{goo}{Googol ShallowMind, New London, Michigan, USA}


\icmlkeywords{Low rank,Nuclear norm minimization,Low level vision}

\vskip 0.3in
]




\begin{abstract}
Low rank matrix approximation (LRMA) has drawn increasing attention in recent years, due to its wide range of applications in computer vision and machine learning. However, LRMA, achieved by nuclear norm minimization (NNM), tends to over-shrink the rank components with the same threshold and ignore the differences between rank components. To address this problem, we propose a flexible and precise model named multi-band weighted $l_p$ norm minimization (MBWPNM). The proposed MBWPNM not only gives more accurate approximation with a Schatten $p$-norm, but also considers the prior knowledge where different rank components have different importance. We analyze the solution of MBWPNM and prove that MBWPNM is equivalent to a non-convex $l_p$ norm subproblems under certain weight condition, whose global optimum can be solved by a generalized soft-thresholding algorithm. We then adopt the MBWPNM algorithm to color and multispectral image denoising. Extensive experiments on additive white Gaussian noise removal and realistic noise removal demonstrate that the proposed MBWPNM achieves a better performance than several state-of-art algorithms.
\end{abstract}

\section{Introduction}
\label{submission}

Image noise is an undesirable by-product of image capture and severely damages the quality of acquired images. Removing noise is a vital work in various computer vision tasks, e.g., recognition \cite{he2016deep} and segmentation \cite{shi2000normalized}. Image denoising aims to recover the clean image x from its noisy observation $ y=x+n $, where n is generally assumed to be additive white Gaussian noise (AWGN). Numerous effective methods have been proposed in the past decades, and they can be categorized into filter based methods \cite{dabov2008image,buades2005non}, sparse coding based methods \cite{elad2006image,Dong2013}, low-rankness based methods \cite{gu2014weighted,gu2017weighted}, convolutional neural nets (CNNs) based methods \cite{Zhang2017,chen2017trainable,mao2016image,Zhang2018}, etc. In this paper, we focus on the method of LRMA with regularization.

LRMA problem can be solved by the low rank matrix factorization and the rank minimization. The latter restores the data matrix through attaching a rank constraint on the estimated matrix, formulated by the following objective function:
\begin{equation}
\mathop {\min }\limits_X \left\| {X - Y} \right\|_F^2\;,\;s.t.\;rank(X) \le r,
\end{equation}
where ${\left\|  \cdot  \right\|_F}$ denotes the Frobenius norm, and $r$ is the given rank. However, rank minimization is NP hard and is tough to solve directly. As the tightest convex relaxation of the rank function \cite{fazel2002matrix}, the nuclear norm minimization (NNM) has been generally served as an ideal approximation for regularization term. The nuclear norm is defined as the sum of the singular values, i.e.,
${\left\| X \right\|_ * } = {\rm{ }}\sum\nolimits_i {{\sigma _i}} (X)$, where ${\sigma _i}(X)$ represents the i-th singular value of a given matrix $X \in {R^{m \times n}}$. Therefore, the NNM model aims to find a low rank matrix $X$ from the degraded observation $Y$ which satisfies the following energy function:
\begin{equation}
X = \arg \mathop {\min }\limits_X \left\| {X - Y} \right\|_F^2 + \lambda {\left\| X \right\|_*}
\end{equation}
where $\lambda$ is a parameter to balance the data fidelity term and regularization. A soft-thresholding operation \cite{cai2010singular} is presented to solve NNM with the theoretical guarantees, that Candes and Recht \cite{candes2009exact} proved that low rank matrix can be efficiently recovered by dealing with NNM problems. However, most of the NNM based models treat all singular values equally, this ignores the prior knowledge in many image inverse problems, such as major edge and texture information expressed by larger singular values. To address this problem, Zhang et al. introduced a Truncated Nuclear Norm (TNN) regularization \yrcite{zhang2012matrix} to improve the ability of preserving the textures while Sun et al. \yrcite{sun2013robust} proposed a Capped Nuclear Norm (CNN) regularizer. A similar regularizer, namely Partial Sum Nuclear Norm (PSNN), has been developed in \cite{oh2016partial}, which only minimizing the smallest $N-r$ singular values, where $N$ is the number of singular values and $r$ is the rank of matrix. Meanwhile, to improve the flexibility of the nuclear norm, Gu et al. generalized NNM to the weighted NNM (WNNM)\yrcite{gu2014weighted} with singular values attached to different weights, defined as:
\begin{equation}
{\left\| X \right\|_{\omega ,*}} = \sum\nolimits_i {{\omega _i}{\sigma _i}(X)}
\end{equation}
where $\omega  = [{\omega _1},...,{\omega _{\min \{ m,n\} }}]({\omega _i} \ge 0)$ is the weight vector of $X$. Indeed, TNN, CNN and PSNN can be considered as the special case of WNNM, fixing weight vector with 0 and 1. Moreover, some efforts surrounded WNNM show an outstanding experimental performance in various image processing applications such as image denoising \cite{gu2014weighted}, background subtraction \cite{gu2017weighted}, image deblurring \cite{ren2016image}, inpainting \cite{yair2018multi}. In particular, a multil-channel optimization model for real color image denoising under the WNNM framework (MCWNNM) \cite{xu2017multi} is proposed, which shows that by considering different contribution of the R, G and B channels based on their noise levels, both noise characteristics and channel correlation can be effectively exploited.

Additionally, the standard NNM tends to over-shrink the singular values with the same threshold, resulting in a deviation far from the original solution of rank minimization problem. The problem, obtained suboptimal solution by WNNM under certain conditions, is also discovered in \cite{lu2015generalized}. More specifically, Lu et al. provided a simple proof and a counterexample, which is better compared with WNNM. Therefore, Schatten $p$-norm is defined as the $l_p$ norm of the singular values $(\sum\nolimits_i {\sigma _i^p{)^{\frac{1}{p}}}} $ with $0 < p \le 1$ and proposed \cite{nie2012low,xie2016weighted} to impose low rank regularization, achieving a more accurate recovery of the data matrix while requiring only a weaker restricted isometry property \cite{liu2014exact}. As shown in the experimental section, the results also demonstrate that the Schatten $p$-norm based model performs significantly better than the WNNM.

Inspired by the $l_p$ norm and MCWNNM\cite{xu2017multi}, in this paper, we propose a new model for color and multispectral image denoising via non-convex multi-band weighted $l_p$ norm minimization (MBWPNM). Introducing Schatten $p$-norm for weighted nuclear norm makes the problem much more troublesome as well as the MCWNNM does not has an analytical solution. Consequently, for the weighted  $l_p$ norm regularization term of the low rank model, we find that it is equivalent to a non-convex $l_p$ norm subproblems. We further apply a generalized soft-thresholding algorithm (GST)\cite{zuo2013generalized}to solve the low rank model with weighted $l_p$ norm. Meanwhile, rigorous mathematical proof of the equivalence and complexity analysis are presented in the later section. The contributions of our work are summarized as follows:
\begin{itemize}
\item We propose a new model:MBWPNM, and present an efficient optimization algorithm to solve MBWPNM.
\item We adopt the MBWPNM model to color and multispectral image denoising, and validate robustness of our model using synthetic and real images. The experimental results indicate that the MBWPNM not only outperforms many state-of-art denoising algorithm both quantitatively and qualitatively, but also results in a competitive speed than MCWNNM.
\end{itemize}

\section{Related Work}

As a generalization to the weighted nuclear norm minimization (WNNM) model, the multi-channel weighted nuclear norm minimization (MCWNNM) model is defined as:
\begin{equation}
X = \arg \mathop {\min }\limits_X \left\| {W(X - Y)} \right\|_F^2 + {\left\| X \right\|_{\omega \; * }}
\end{equation}
with $W=\left( {\begin{array}{*{20}{c}}
	{\sigma _r^{ - 1}I}&0&0\\
	0&{\sigma _g^{ - 1}I}&0\\
	0&0&{\sigma _b^{ - 1}I}
	\end{array}} \right)$,
where ${\rm{\{ }}{\sigma _r},{\sigma _g},{\sigma _b}{\rm{\} }}$ is standard deviation in each channel and $I$ is the identity matrix. By introducing an augmented variable $Z$, the MCWNNM model can be reformulated as a linear equality-constrained problem with two variables and its augmented Lagrangian function is:
\begin{multline}
L(X,Z,A,\rho ) = \arg \mathop {\min }\limits_X \left\| {W(X - Y)} \right\|_F^2 + {\left\| Z \right\|_{\omega  * }} \\
+ \left\langle {A,X - Z} \right\rangle  + \frac{\rho }{2}\left\| {X - Z} \right\|_F^2
\label{mcwnnm}
\end{multline}
where $A$ is the augmented Lagrangian multiplier and $\rho>0$ is the penalty parameter. According to \cite{xu2017multi}, the problem \eqref{mcwnnm} can be solved by alternating direction method of multipliers (ADMM)\cite{boyd2011distributed} framework.

The ADMM is an algorithm that solves convex optimization problems by breaking them into smaller pieces, each of which are then easier to handle. In particular, there is widespread current interest applying ADMM in many of problems involving $l_1$ norms across statistics, machine learning and signal processing. However, on one side, because of multiple iterations, it is often time-consuming to get a relatively accurate solution using ADMM; on the other side, preferences such as $\rho$ play an important role for convergence. Theoretically, the MBWPNM also can be solved by ADMM but its cost is higher than MCWNNM due to the complexity. In this paper, we present a more effective algorithm, which achieves better results with less time.

\section{Multi-band Weighted $l_p$ norm minimization}

\subsection{Problem Formulation}

Given a matrix $Y$, the proposed optimization model under the MCWNNM framework, which aims to find a matrix $X \in {R^{m \times n}}$ closing to $Y \in {R^{m \times n}}$ as much as possible, is defined as:
\begin{equation}
\hat X = \arg \;\mathop {\min }\limits_X \left\| {W(X - Y)} \right\|_F^2 + \left\| X \right\|_{\omega ,{S_p}}^p
\label{mbwpnm}
\end{equation}
involving two terms: a $F$-norm data fidelity and a weighted Schatten $p$-norm regularization. Then the weighted Schatten $p$-norm of matrix $X$ with power $p$ is
\begin{equation}
\left\| X \right\|_{\omega ,{S_p}}^p = \sum\nolimits_{i = 1}^{\min \{ m,n\} } {{\omega _i}} \sigma _i^p(X)
\end{equation}
where $\omega  = [{\omega _1},...,{\omega _{\min \{ m,n\} }}]$ is a non-negative weight vector, ${\sigma _i}$ is the $i$-th singular value of $X$, and $0 < p \le 1$. Note that MCWNNM is a special case of MBWPNM when power $p$ is set to 1. In the next subsection, we will launch a detailed discussion of the optimization above MBWPNM.

\subsection{Optimization of MBWPNM}

We first give the following theorem before analyzing and discussing the optimization of MBWPNM.
\newtheorem{mythm}{Theorem}
\begin{mythm}\cite{horn1990matrix}
Let $A,B\in{M_{m,n}}$ be given and let $q = \min \{ m,n\} $, the following inequalities hold for the decreasingly ordered singular values of $A$ and $B$:
\begin{equation}
\begin{gathered}
{\sigma _i}(AB) \le {\sigma _i}(A)\left\| B \right\|,\\
and\;\left\| B \right\| = {\sigma _1}(B),for\;1 \le i \le q.
\end{gathered}
\end{equation}
\end{mythm}
Introducing the Schatten $p$-norm makes the problem \eqref{mbwpnm} more complicated to solve directly than the original MCWNNM model. Changing a way of thinking, we consider the following object function similar to problem \eqref{mbwpnm}:
\begin{equation}
\hat X = \arg \;\mathop {\min }\limits_X \left\| {W(X - Y)} \right\|_F^2 + \lambda \left\| {WX} \right\|_{\omega ,{S_p}}^p.
\label{mbwpnm1}
\end{equation}

\begin{mythm}
Problem \eqref{mbwpnm} is equivalent to problem \eqref{mbwpnm1}, that they share the same solution.
\end{mythm}
\begin{proof}
The proof can be found in the supplementary material.
\end{proof}
Thus, the original problem \eqref{mbwpnm} has been converted to problem \eqref{mbwpnm1}, which can be solved more easily. We then have the Theorem 3 and 4.
\begin{mythm}\cite{hom1991topics}
Let $A,B \in {M_{m,n}}$, let $q = \min \{ m,n\} $, and let ${\sigma _1}(A) \ge ... \ge {\sigma _q}(A)\;and\;{\sigma _1}(B) \ge ... \ge {\sigma _q}(B)$, denote the non-increasingly ordered singular values of $A$ and $B$, respectively. Then 
\begin{equation}
tr({A^T}B) \le \sum\limits_{i = 1}^q {{\sigma _i}(A){\sigma _i}(B) = } \;tr(\sigma {(A)^T}\sigma (B)).
\end{equation}
\end{mythm}
\begin{mythm}
Let the SVD of $WY \in {R^{m \times n}}$ be $WY = U\Delta {V^T}$ with $\Delta \; = diag({\delta _1},...{\delta _r}),r = \min \{ m,n\} $. Then an optimal solution to \eqref{mbwpnm1} is $X = {W^{ - 1}}U\Sigma {V^T}$ with $\Sigma  = diag({\sigma _1},...{\sigma _r})$, where ${\sigma _i}$  is given by solving the problem below:
\begin{equation}
\begin{gathered}
\mathop {\min }\limits_{{\sigma _1}, \ldots {\sigma _r}} \sum\limits_{i = 1}^r {\left[ {{{({\sigma _i} - {\delta _i})}^2} + \lambda {w_i}\sigma _i^p} \right]} \;,\;i = 1, \ldots r\\
s.t.\;{\sigma _i} \ge 0,\;and\;{\sigma _i} \ge {\sigma _j},\;for\;i \le j
\label{lp}
\end{gathered}
\end{equation}
\end{mythm}
\begin{proof}
	The proof can be found in the supplementary material.
\end{proof}
Obviously, problem \eqref{lp} will be more trivial to handle if the additional order constraint (i.e. ${\sigma _i} \ge {\sigma _j},\;i \le j$) can be abandoned. Therefore, we take into account the more general condition that problem \eqref{lp} is decomposed into $r$ independent subproblems:
\begin{equation}
\min {f_i}(\sigma ) = {({\sigma _i} - {\delta _i})^2}\; + \lambda {\omega _i}\sigma _i^p,\;i = 1, \ldots r,
\label{sublp}
\end{equation}
which has been explored in \cite{lu2015generalized,zuo2013generalized}. To obtain the solution of each subproblem \eqref{sublp}, the generalized soft-thresholding (GST) algorithm \cite{zuo2013generalized}is adopted. Specifically, given ${\omega _i}$ and $p$, there exists a specific threshold:
\begin{equation}
\tau _p^{GST}({w_i}) = {(2{\omega _i}(1 - p))^{\frac{1}{{2 - p}}}} + {\omega _i}p{(2{\omega _i}(1 - p))^{\frac{{p - 1}}{{2 - p}}}}.
\end{equation}
If ${\delta _i} < \tau _p^{GST}({\omega _i}),{\sigma _i} = 0$ is the global minimum; otherwise, for any ${\delta _i} \in (\tau _p^{GST}({\omega _i}), + \infty ),{\sigma _i}$ has one unique minimum $S_p^{GST}({\delta _i};{\omega _i})$ satisfying the problem \eqref{sublp}, which can be obtained by solving the following equation:
\begin{equation}
S_p^{GST}({\delta _i};{\omega _i}) - {\delta _i} + {\omega _i}p{(S_p^{GST}({\delta _i};{\omega _i}))^{p - 1}} = 0.
\end{equation}
The complete description of the GST algorithm is shown in Algorithm 1; please refer to \cite{zuo2013generalized} for more details about the GST algorithm.
\begin{algorithm}
	\caption{Generalized Soft-Thresholding(GST)}
	\label{alg1}
	\begin{algorithmic}[1]
		\STATE {\bfseries Input:} $s,w,p,J$
		\STATE {\bfseries Output:} $S_p^{GST}(\delta ;\omega )$
		\STATE $\tau _p^{GST}(\omega ) = {(2\omega (1 - p))^{\frac{1}{{2 - p}}}} + \omega p{(2\omega (1 - p))^{\frac{{p - 1}}{{2 - p}}}};$
		\IF{${\rm{|}}\delta {\rm{|}} \le \tau_p^{GST}(\omega)$}
		\STATE$S_p^{GST}(\delta ;\omega ) = 0;$
		\ELSE
		\STATE$k = 0,\;{\sigma ^{(k)}} = |\delta |;$
		\FOR{$k=0$ {\bfseries to} $J$}
		\STATE${\sigma ^{(k + 1)}} = \;|\delta | - \omega p{({\sigma ^{(k)}})^{p - 1}};$
		\STATE$k = k + 1;$
		\ENDFOR
		\STATE$S_p^{GST}(\delta ;\omega ) = {\mathop{\rm sgn}} (\delta ){\sigma ^{(k)}};$
		\ENDIF
		\STATE {\bfseries Return } $S_p^{GST}(\delta ;\omega ).$
	\end{algorithmic}
\end{algorithm}

\subsection{Optimal Solution under Non-Descending Weights}
Generally, the larger singular values usually contain the major edge and texture information, which implies that the small singular values should be penalized more than the large ones. Getting back to the optimization of problem \eqref{lp}, it is meaningful to assign the non-descending weights to the non-ascending singular values for most practical applications in low level vision. Based on such considerations, we have the Theorem 5.
\begin{mythm}
The optimal solutions ${\sigma _i}$ of all the independent subproblems in \eqref{sublp} satisfy the following inequality:
\begin{equation}
\begin{gathered}
{\sigma _1} \ge {\sigma _2} \ge  \ldots  \ge {\sigma _r}\;,\;\\
when\;0 \le {\omega _1} \le {\omega _2} \le  \ldots {\omega _r}\;\;and\;\;\;{\delta _1} \ge {\delta _2} \ge  \ldots  \ge {\delta _r}.
\end{gathered}
\end{equation}
\end{mythm}
\begin{proof}
For $i \le j \le r$, we have the following inequality:
\begin{equation}
\begin{gathered}
f({\sigma _j},{\delta _i},{\omega _i}) \ge f({\sigma _i},{\delta _i},{\omega _i})\\
f({\sigma _i},{\delta _j},{\omega _j}) \ge f({\sigma _j},{\delta _j},{\omega _j})
\end{gathered}
\end{equation}
after substitution i.e.
\begin{equation}
\begin{gathered}
{({\sigma _j} - {\delta _i})^2} + \lambda {\omega _i}\sigma _j^p \ge {({\sigma _i} - {\delta _i})^2} + \lambda {\omega _i}\sigma _i^p\\
{({\sigma _i} - {\delta _j})^2} + \lambda {\omega _j}\sigma _i^p \ge {({\sigma _j} - {\delta _j})^2} + \lambda {\omega _j}\sigma _j^p
\end{gathered}
\end{equation}
Summing them together and simplification, it reduces to
\begin{equation}
2({\sigma _i} - {\sigma _j})({\delta _i} - {\delta _j}) + \lambda (\sigma _i^p - \sigma _j^p)({\omega _j} - {\omega _i}) \ge 0
\end{equation}
Thus
${\sigma _i} \ge {\sigma _j},\;when\;\;{\omega _j} \ge {\omega _i}\;and\;{\delta _i} \ge {\delta _j}.$
\end{proof}
According to Theorem 5, solving problem \eqref{lp} is equivalent to solving all the independent subproblems in \eqref{sublp} with the non-descending weights. So far, the original problem \eqref{mbwpnm} undergo a series of transformations:\eqref{mbwpnm}$\Rightarrow$\eqref{mbwpnm1}$\Rightarrow$\eqref{lp}$\Rightarrow$\eqref{sublp} before it can be effectively solved. Finally, the proposed algorithm of MBWPNM is summarized in Algorithm 2.
\begin{algorithm}
	\caption{MBWPNM via GST}
	\label{alg2}
	\begin{algorithmic}[1]
		\STATE {\bfseries Input:} ${{Y,W,\{ }}{\omega _i}{{\} }}_{i = 1}^r$ {in non-descending order} $p$
		\STATE {\bfseries Output:} {Matrix} $\hat{X}$
		\STATE ${{WY = U}}\Sigma {{{V}}^T},\Sigma  = diag({\delta _1}, \ldots ,{\delta _r});$
		\FOR{$i=1$ {\bfseries to} $r$}
		\STATE ${\sigma _i} = GST({\delta _{i,}}{\rm{,}}{\omega _i},p);$
		\ENDFOR
		\STATE $\Sigma  = diag({\sigma _{1,}} \ldots ,{\sigma _r});$
		\STATE {\bfseries Return} $\hat X = {W^{ - 1}}U\Delta {V^T}.$
	\end{algorithmic}
\end{algorithm}

\section{Color and Multispectral Image Denoising}
\subsection{The Denoising Algorithm}
Similar to MCWNNM model, MBWPNM is adopted to the matrix of color image nonlocal similar patches for noise removal, which are concatenated from RGB channel. Specifically, given a degraded color image $Y$, each local patch of size $h \times h \times 3$ is stretched to a patch vector, denoted by ${y_i} = {[y_{i,r}^T,y_{i,g}^T,y_{i,b}^T]^T} \in {R^{3{{\rm{h}}^2}}}$, where ${y_{i,r}},{y_{i,g}},{y_{i,b}} \in {R^{{h^2}}}$ are the patches in R,G,B channels, respectively. For a local patch $y_i$, we search for the $M-1$ most similar patches across the image $Y$ (in practice, in a large enough local window) by the block matching method proposed in \cite{dabov2008image}. Then, by stacking those nonlocal similar patches into a matrix $Y_i$ column by column, we have ${Y_i} = {X_i} + {N_i} \in {R^{3{h^2} \times M}}$, where $X_i$ and $N_i$ are the corresponding clean and noise patch matrices. Applying MBWPNM algorithm to $Y_i$, the whole image $X$ can be estimated through aggregating all the denoised patches $X_i$.

Compared with RGB color image, multispectral image(MSI) is actually a more channels image. So, it is natural to apply MBWPNM model to the area of MSI denoising. For a given 3-D patch cube, which is stacked by patches at the same position of MSI over all band, there are also many patch cubes similar to it. The low rank property of the image matrix is explored by rearranging respectively those nonlocal similar patch cubes as a 1-D vector and stacking those vectors into a matrix column by column. Figure~\ref{mil_image}. shows the detailed process. \cite{chang2017hyper} pointed out that the spectral and non-local similarity information from mutually complementary priors can be jointly utilized in this way and effectively improve the performance of denoising. It is also reasonable that using a weight matrix $W$ adjusts the contributions of multiple bands (channels) based on their different noise levels in real scene, where $W = \left( {\begin{array}{*{20}{c}}
	{\sigma _1^{ - 1}I}&0&0\\
	0& \ddots &0\\
	0&0&{\sigma _B^{ - 1}I}
	\end{array}} \right)$, $I$ is the identity matrix and $B$ is the number of bands.

In the application of color image and MSI denoising, theoretically, the larger the singular values, the less they should be shrunk. Therefore, the weight assigned to the singular value ${\sigma _i}(X)$ is inversely proportional to it, and we let
\begin{equation}
{\omega _i} = c\sqrt M /(\sigma _i^{1/p} + \varepsilon ),
\end{equation}
where $c>0$ is a constant, $M$ is the number of similar patches and $\varepsilon$ is to avoid dividing by zero. We adopt the iterative regularization scheme to restore clean image, and then the image $X$ can be reconstructed by aggregating all the denoised patches together. The MBWPNM based color image and MSI denoising algorithm is summarized in Algorithm 3.

\begin{figure}[t]
	\vskip 0.2in
	\begin{center}
		\centerline{\includegraphics[width=\columnwidth]{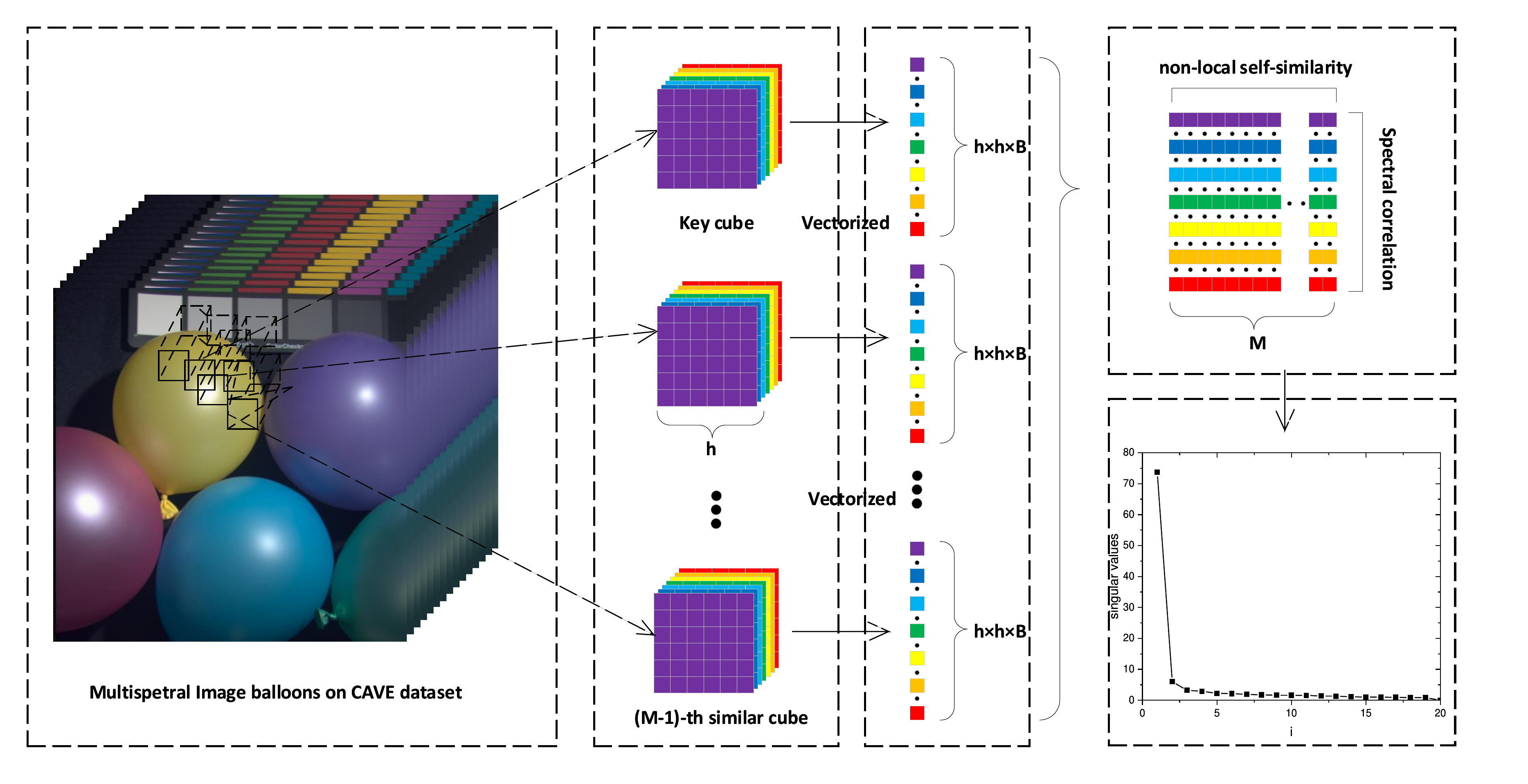}}
		\caption{The flowchart of the low rank property representation.}
		\label{mil_image}
	\end{center}
	\vskip -0.2in
\end{figure}
\begin{algorithm}
	\caption{Image Denoising by MBWPNM}
	\label{alg3}
	\begin{algorithmic}[1]
		\STATE {\bfseries Input:} {Noisy Image} $Y,$ {noise level} ${{{\{ }}{{{\sigma }}_{{i}}}{{\} }}_{{{i = 1}}}^B},K$ 
		\STATE {\bfseries Output:} {Denoised Image} ${\hat X}^K$ 
		\STATE {\bfseries Initialization:} ${\hat X}^0=Y,{\hat Y}^0=Y;$
		\FOR {$K=1$ {\bfseries to} $K$}
		\STATE {Iterative regularization} ${\hat Y}^k = \;{\hat X}^{(k - 1)} + \alpha ({Y} -{\hat X}^{(k - 1)});$
		\FOR {{each local patch} $y_i$}
		\STATE {Find no local similar patches to form matrix } ${{Y}}_i^k;$
		\STATE {Apply the MBWPNM model to} ${{Y}}_i^k$ {Algorithm 2;}
		\STATE {Get the estimated} $\hat X_i^k;$
		\ENDFOR
		\STATE {Aggregate} $\hat X_i^k$ {to form the image} $\hat X^k;$
		\ENDFOR
		\STATE {\bfseries Return } {Denoised Image} $\hat X^k.$
	\end{algorithmic}
\end{algorithm}
\subsection{Complexity Analysis}
For a matrix of size ${h^2}B \times M$, where $B$ is the number of bands (channels), $h$ denotes the width and height of local patch $y_i$ in Algorithm 3, and $M$ is the number of similar patches, the cost of SVD in each iteration(step 6) is $O(\min \{ {h^4}{B^2}M,{h^2}B{M^2}\} )$. The GST algorithm (step 3 in Algorithm 2) costs $O(JM)$, where $J$ is the number of iterations in the GST Algorithm. Therefore, the overall cost is $O(K \times N \times (\min \{ {h^4}{B^2}M,{h^2}B{M^2}\}  + JM))$, where $K$ is the number of iterations in Algorithm 3, and $N$ denotes the total number of patches. Specially, an amount of $O(K \times N \times (\min \{ {h^4}M,{h^2}{M^2}\}  + JM))$ flops are required for the color image denoising.
\section{Experimental Results}
\subsection{Experimental Setting}
Firstly, we compare MBWPNM method with several state-of-art image denoising methods both simulated and real noisy condition, including CBM3D\cite{dabov2007color}, NCSR\cite{Dong2013}, EPLL\cite{zoran2011learning}, Guided\cite{xu2018external}, DnCNN\cite{Zhang2017}, FDDNet\cite{Zhang2018}, MCWNNM\cite{xu2017multi}. And then we also compare it with 9 MSI denoising algorithms: 1-D sparse representation-based methods (SDS\cite{lam2012denoising}, ANLM\cite{manjon2010adaptive}), 2-D low-rank matrix recovery methods (LRMR\cite{zhang2014hyperspectral}, NMF\cite{ye2015multitask}), state-of-art tensor methods (BM3D\cite{dabov2007image}, LRTA\cite{renard2008denoising}, BM4D\cite{maggioni2013nonlocal}, ISTreg\cite{xie2016multispectral}, LLRT\cite{chang2017hyper}).

In simulated experiments, the noise levels in each band (channel) have been known. We set the noise level as the root mean square (RMS):
\begin{equation}
\sigma  = \sqrt {(\sigma _1^2 + \sigma _2^2 +  \ldots  + \sigma _B^2)/B} \;,B \ge 3
\end{equation}
for the method that requires a single parameter to input such as CBM3D and LLRT. In the real cases, we use the method to estimate \cite{liu2008automatic} the noise level for each channel.
\subsection{Color Image Denoising}
For a fair comparison, we keep the parameter settings of both MCWNNM and MBWPNM to achieve their best performance. More specifically, we set the patch size as $p=6$, the number of non-local similar patches as $M=70$, the window size for searching similar patches as $40\times40$, and the number of iterations is set as $K=8$ in simulated experiments.

\begin{figure}[h]
	\vskip 0.2in
	\begin{center}
		\subfigure[]{\includegraphics[width=0.45\columnwidth]{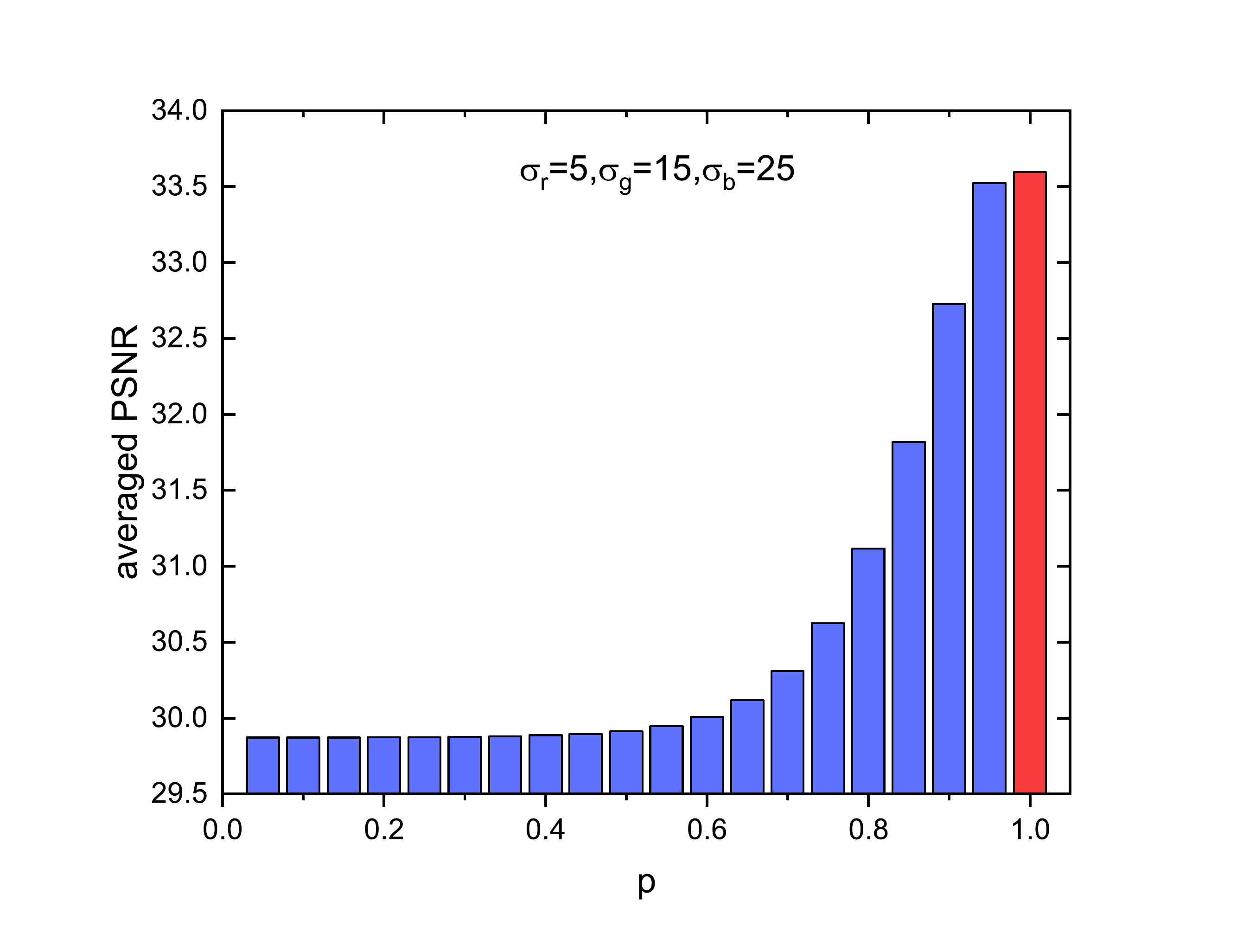}}
		\quad
		\subfigure[]{\includegraphics[width=0.45\columnwidth]{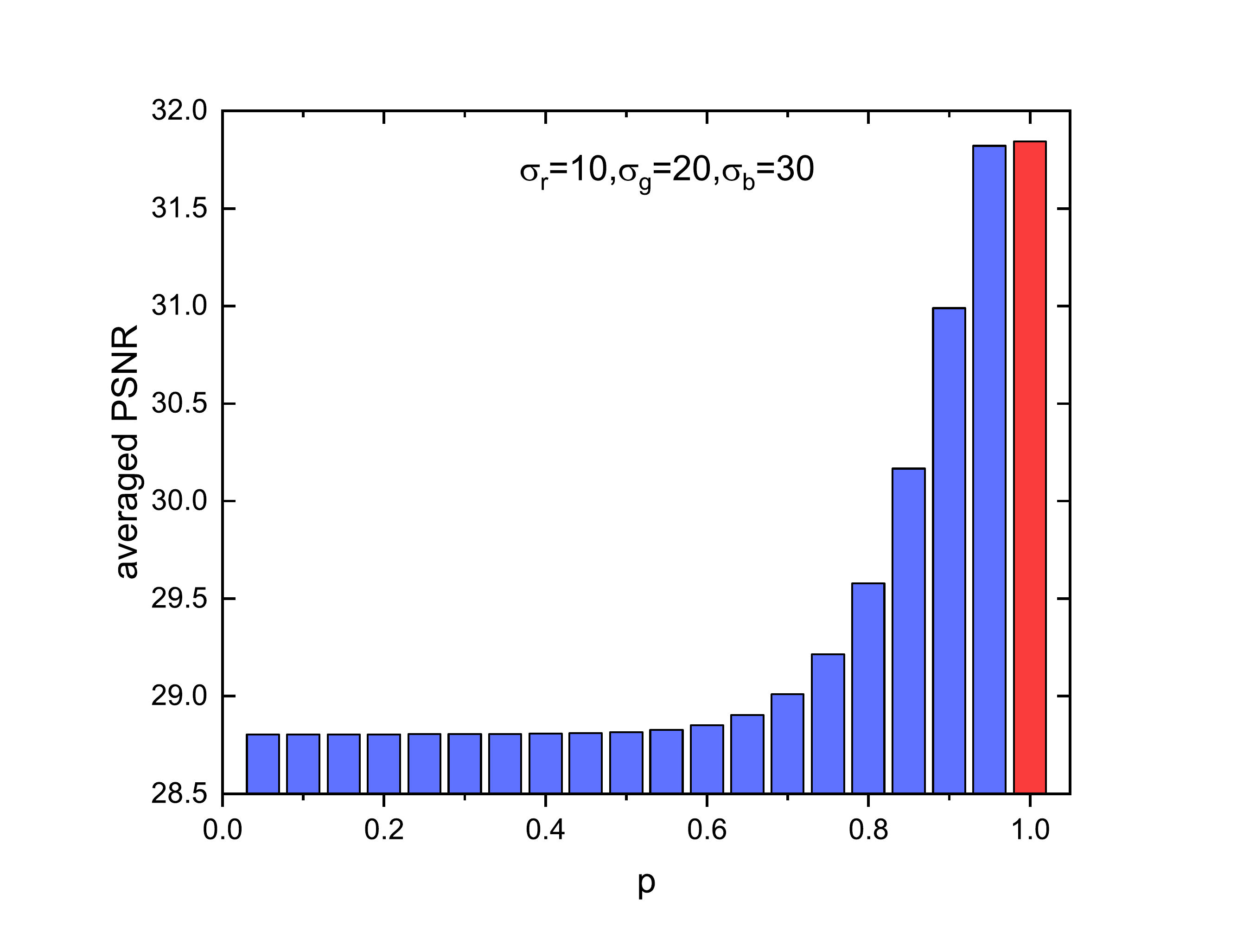}}
		\quad
		\subfigure[]{\includegraphics[width=0.45\columnwidth]{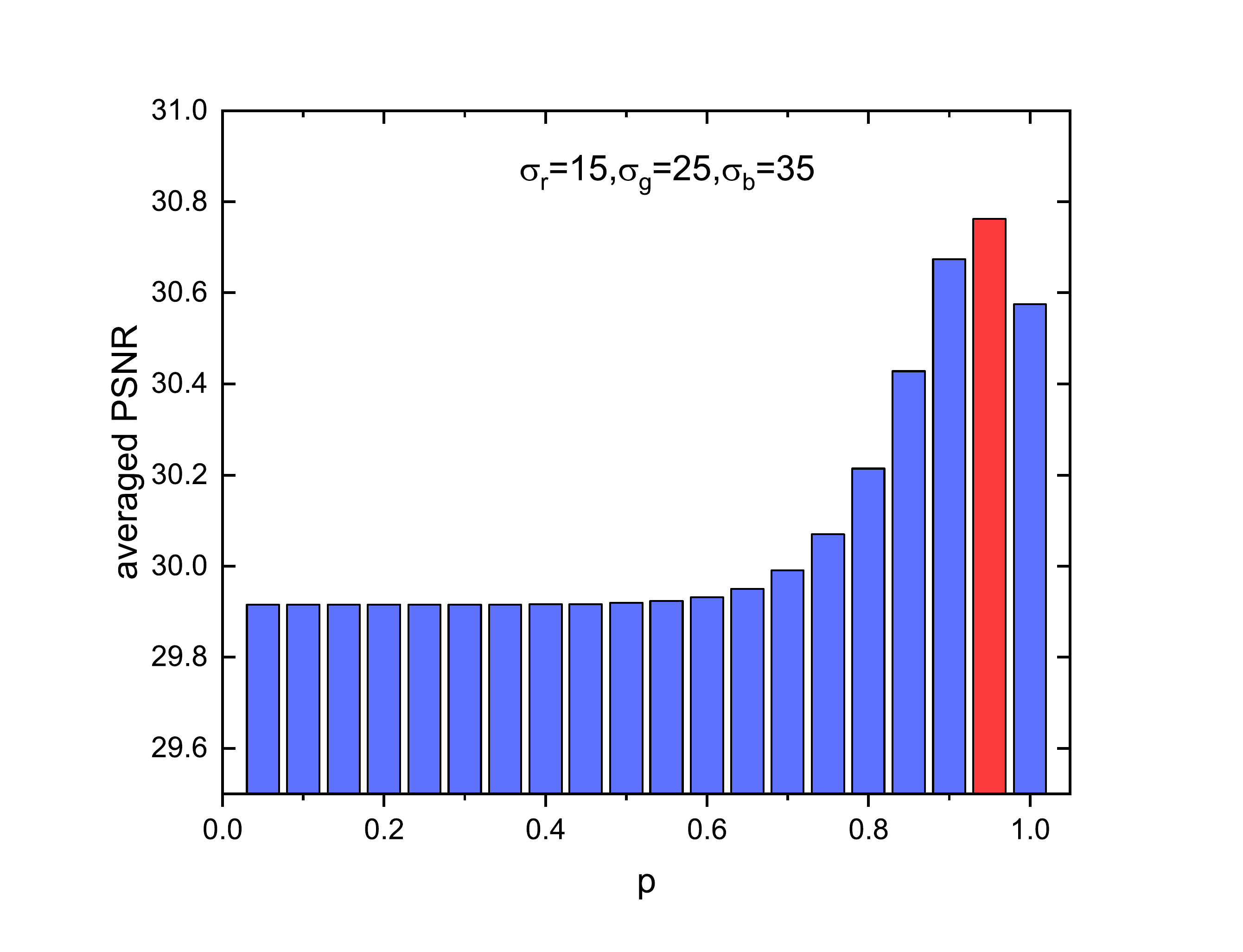}}
		\quad
		\subfigure[]{\includegraphics[width=0.45\columnwidth]{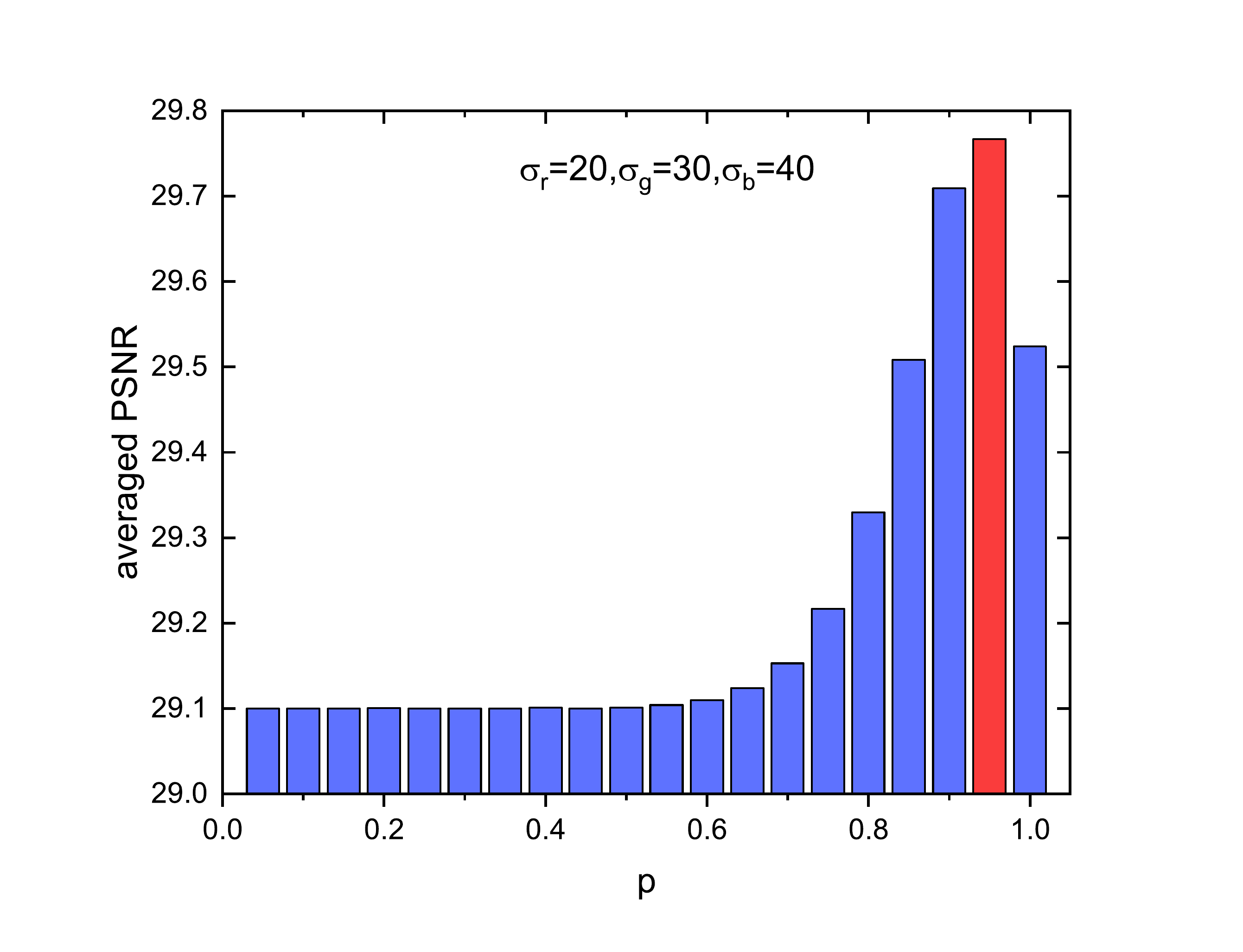}}
		\quad
		\subfigure[]{\includegraphics[width=0.45\columnwidth]{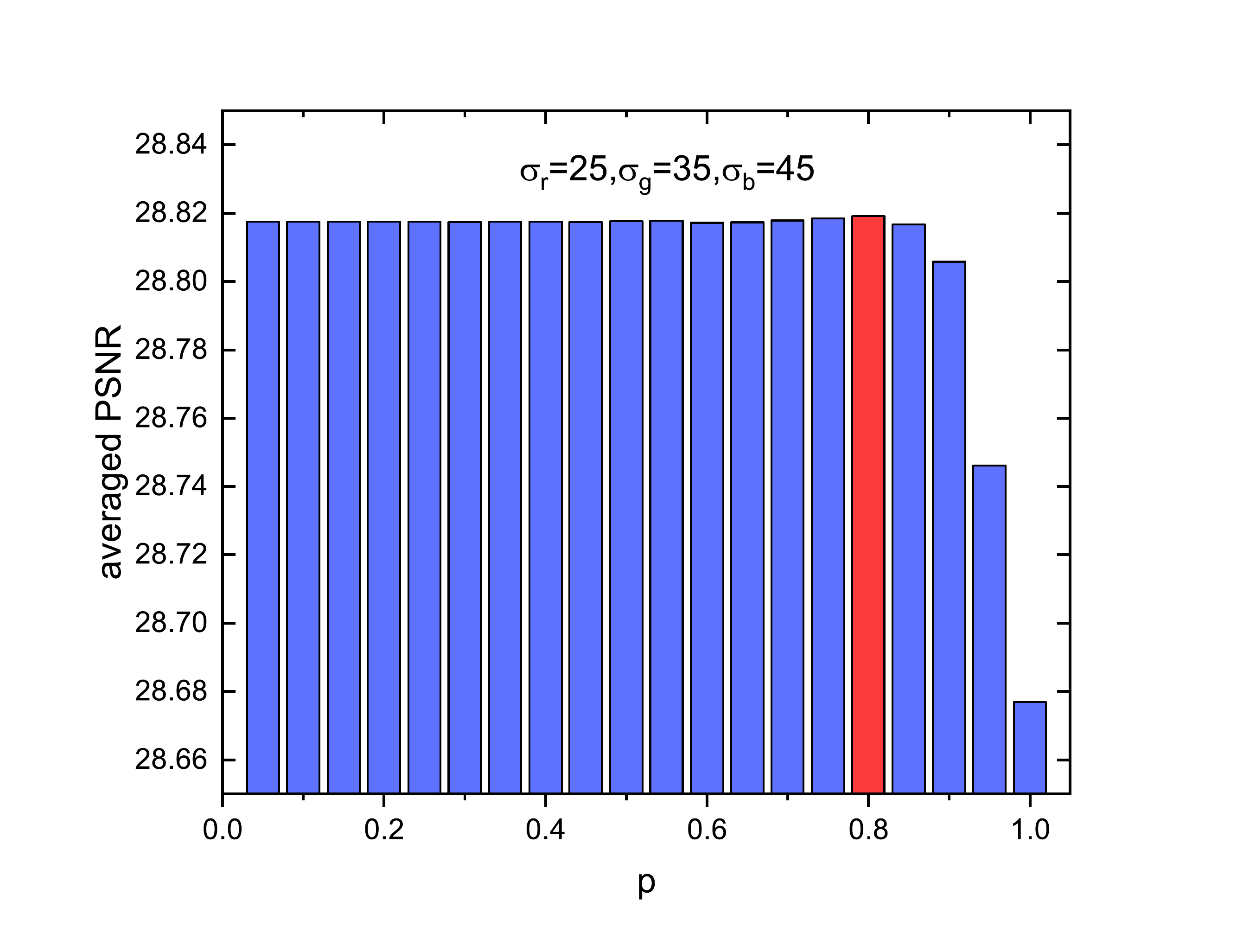}}
		\quad
		\subfigure[]{\includegraphics[width=0.45\columnwidth]{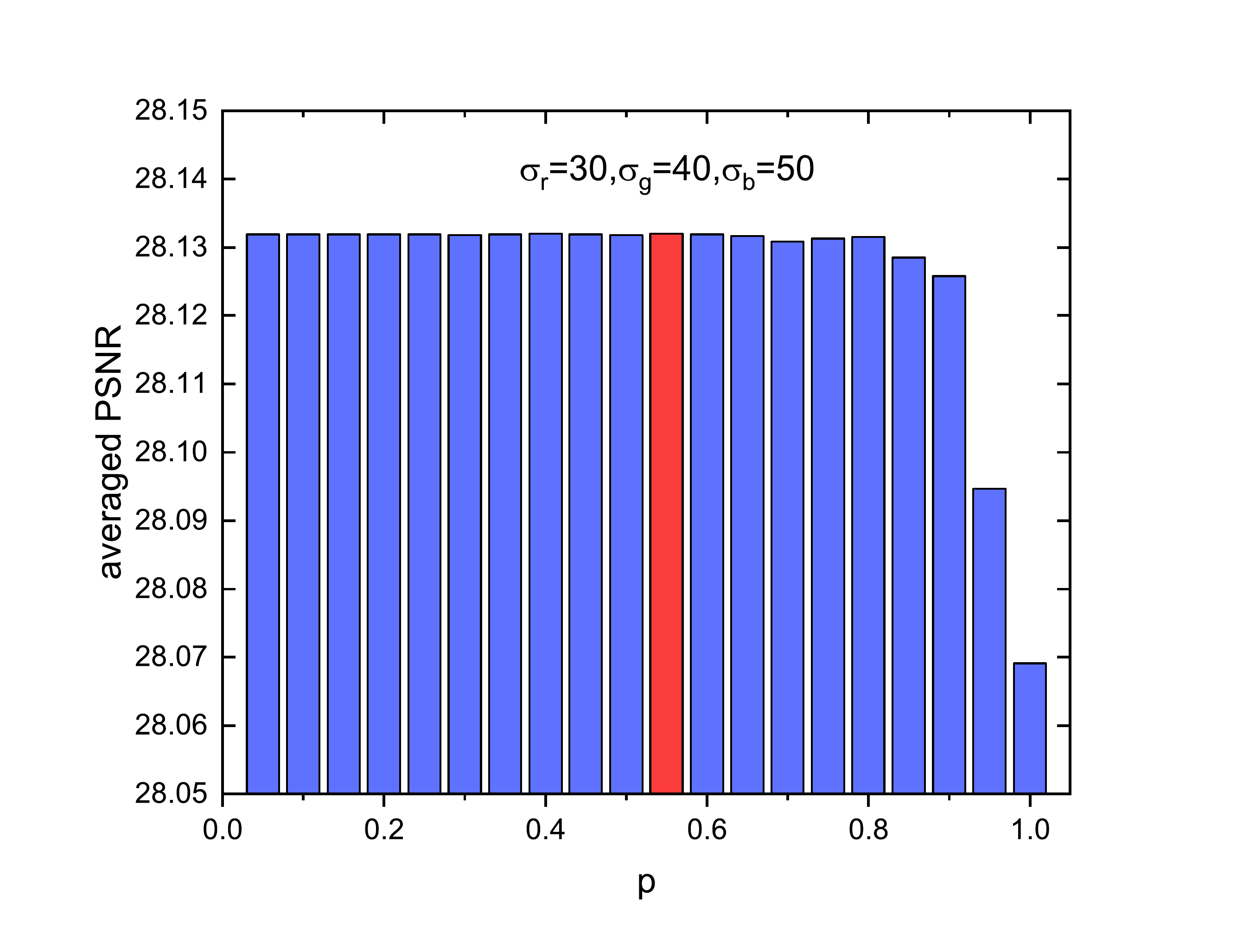}}
		\caption{The PSNR results of changing power p under different noise level on 18 color images from Master dataset\cite{zhang2011color}}
		\label{diff_p}
	\end{center}
	\vskip -0.2in
\end{figure}

\subsubsection{Simulated Color Images Denoising}
Intuitively, the different choice of power $p$ has different influence on shrinkage of singular value. How to choose a suitable setting of power $p$ for each noise level is a little complicated. So, we add AWGN to each of the R, G, B channels on the 18 color sub-images from Master dataset\cite{zhang2011color}, and then make a test with different power $p$ under different noise level. The PSNR results are shown in Figure~\ref{diff_p}, while the best PSNR result for each group is highlighted in red. In this test, we set the value of power $p$ changing from 0.05 to 1 with interval 0.05, and select 6 noise level (5,15,25), (10,20,30), (15,25,35), (20,30,40), (25,35,45), (30,40,50). In the first two subfigures of Figure~\ref{diff_p}, we can observe that the best values of $p$ are 1 when the noise level is low. And the best values of $p$ are reduced to 0.95 in the subsequent two subfigures. With the increase of noise level, the pollution of singular values is much more severe. Consequently, as shown in the last two subfigures of Figure~\ref{diff_p}, the smaller values of $p$ (0.8 and 0.55 for the latter two, respectively) are more suitable for higher noise level. In summary, the optimal value of $p$, by and large, is inversely proportional to the noise level. Sticking to this principle, we did the following experiments.

\begin{table*}[ht]
	\caption{Different simulated results (PSNR) by different methods.}
	\label{table1}
	\vskip 0.15in
	\begin{center}
		\begin{small}
			\begin{sc}
				\begin{tabular}{lccccccc}
					\toprule
					$\sigma_r,\sigma_g,\sigma_b$ & CBM3D & NCSR & EPLL & DnCNN & FFDNet & MCWNNM & MBWPNM \\
					\midrule
					(5,30,15)    & 29.46 & 30.77 & 31.68 & \textbf{33.35} & 33.29 & 32.99 & 33.34\\
					(25,5,30) & 29.38 & 30.16 & 30.82 & 32.56 & 32.48 & 32.15 & \textbf{33.02}\\
					(30,10,50)    & 29.06 & 27.95& 29.05 & 30.65 & 30.44 & 29.81 & \textbf{30.71}\\
					(40,20,30)    & 29.14 & 28.29 & 28.95 & \textbf{31.06} & 30.98 & 29.27 & 29.43\\
					\bottomrule
				\end{tabular}
			\end{sc}
		\end{small}
	\end{center}
	\vskip -0.1in
\end{table*}

\begin{table*}
	\caption{PSNR (dB) and ACT (s) of different methods on 15 cropped real noisy images\cite{nam2016holistic}.}
	\label{table2}
	\vskip 0.15in
	\begin{center}
		\begin{small}
			\begin{sc}
				\begin{tabular}{lccccccc}
					\toprule
					Camera Settings       & CBM3D & NCSR & EPLL & Guided & FFDNet & MCWNNM & MBWPNM \\
					\midrule
					& 38.25 & 38.02 & 37.00 & 40.82 & 37.05 & 41.20 & \textbf{41.21}\\
					Canon 5D III ISO=3200 & 35.85 & 34.76 & 33.88 & 37.19 & 33.91 & \textbf{37.25} & 37.10\\
					& 34.12 & 34.91 & 33.83 & 36.92 & 33.86 & 36.48 & \textbf{37.14}\\
					\midrule
					& 33.10 & 33.51 & 33.28 & 35.32 & 33.31 & \textbf{35.54} & 35.49\\
					Nikon D600 ISO=3200   & 35.57 & 34.13 & 33.77 & 36.62 & 33.81 & 37.03 & \textbf{37.06}\\
					& \textbf{40.77} & 35.44 & 34.93 & 38.68 & 34.98 & 39.56 & 39.56\\
					\midrule
					& 36.83 & 35.98 & 35.47 & 38.88 & 35.50 & \textbf{39.26} & 39.19\\
					Nikon D800 ISO=1600   & 40.19 & 36.39 & 35.71 & 40.66 & 35.75 & \textbf{41.45} & 41.44\\
					& 37.64 & 35.34 & 34.81 & 39.20 & 34.83 & \textbf{39.54} & 39.52\\
					\midrule
					& \textbf{39.72} & 33.63 & 33.26 & 37.92 & 33.30 & 38.94 & 39.10\\
					Nikon D800 ISO=3200   & 36.74 & 33.13 & 32.89 & 36.62 & 32.94 & \textbf{37.40} & 37.10\\
					& \textbf{40.96} & 33.43 & 32.91 & 37.64 & 32.94 & 39.42 & 39.47\\
					\midrule
					& 34.63 & 30.09 & 29.63 & 33.01 & 29.65 & 34.85 & \textbf{34.92}\\
					Nikon D800 ISO=6400   & 32.95 & 30.35 & 29.97 & 32.93 & 30.00 & \textbf{33.97} & 33.96\\
					& 33.61 & 30.12 & 29.87 & 32.96 & 29.88 & 33.97 & \textbf{34.09}\\
					\midrule
					Average 			  & 36.73 & 33.95 & 33.41 & 37.02 & 33.45 & 37.72 & \textbf{37.76}\\
					Time 				  & 6 & 1434 & 424 & 48 & \textbf{4} & 192 & 41\\
					\bottomrule
				\end{tabular}
			\end{sc}
		\end{small}
	\end{center}
	\vskip -0.1in
\end{table*}

\begin{figure*}[ht]
	\vskip 0.2in
	\begin{center}
		\subfigure[Ground Truth]{\includegraphics[width=0.36\columnwidth]{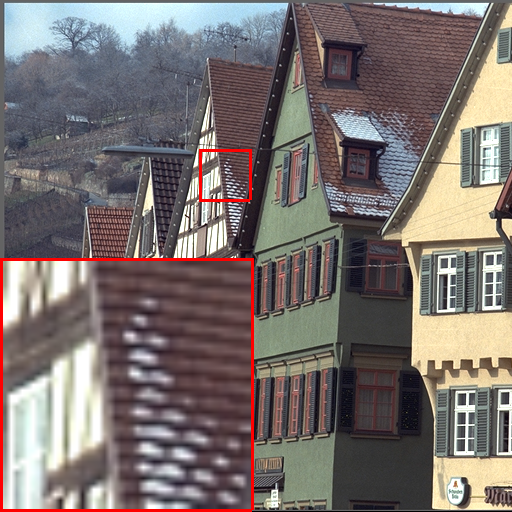}}
		\quad
		\subfigure[Noisy Image]{\includegraphics[width=0.36\columnwidth]{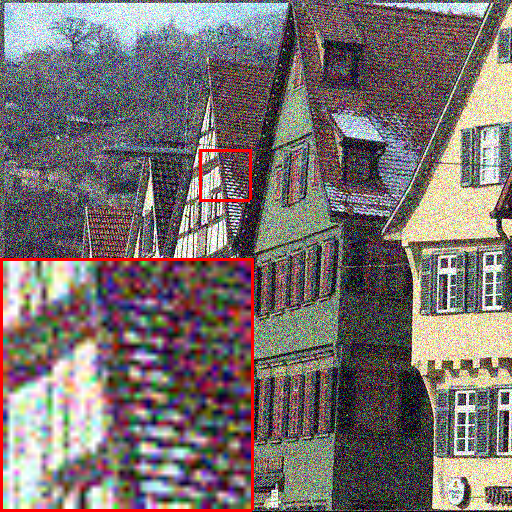}}
		\quad
		\subfigure[CBM3D]{\includegraphics[width=0.36\columnwidth]{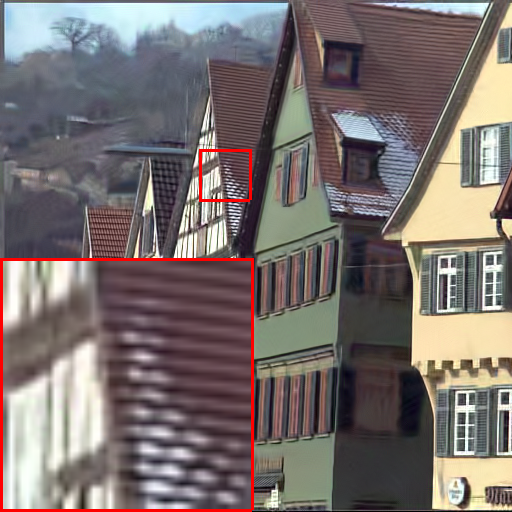}}
		\quad
		\subfigure[NCSR]{\includegraphics[width=0.36\columnwidth]{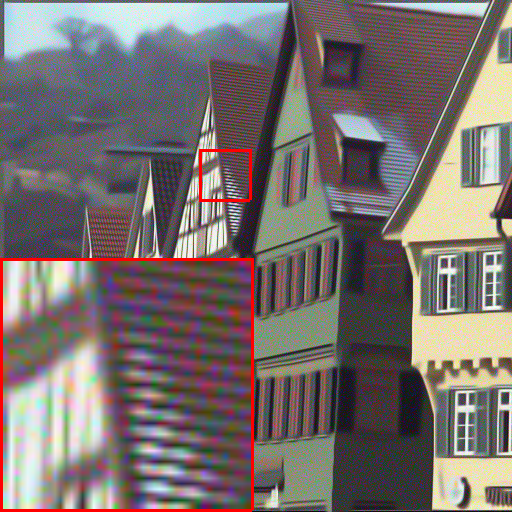}}
		\quad
		\subfigure[EPLL]{\includegraphics[width=0.36\columnwidth]{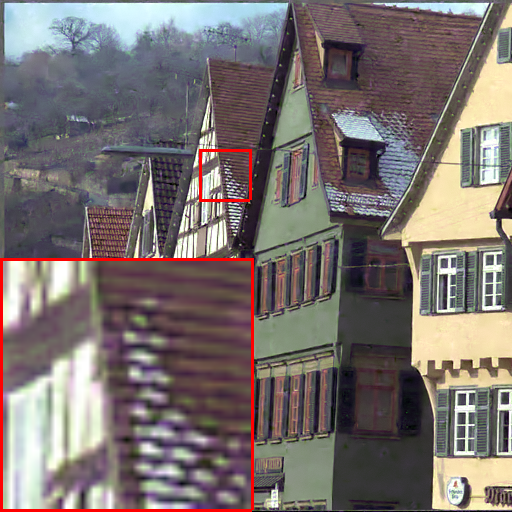}}
		\quad
		\subfigure[DnCNN]{\includegraphics[width=0.36\columnwidth]{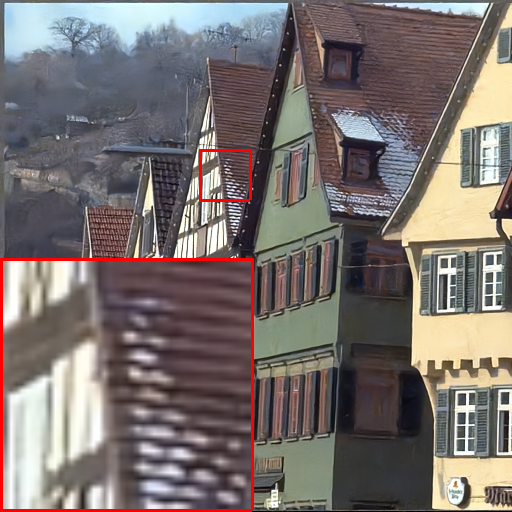}}
		\quad
		\subfigure[FFDNet]{\includegraphics[width=0.36\columnwidth]{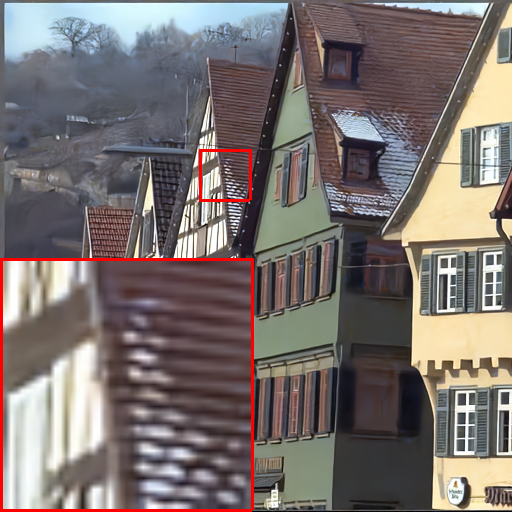}}
		\quad
		\subfigure[MCWNNM]{\includegraphics[width=0.36\columnwidth]{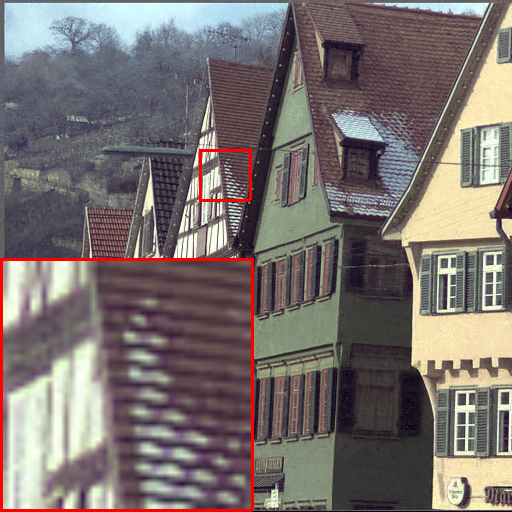}}
		\quad
		\subfigure[MBWPNM]{\includegraphics[width=0.36\columnwidth]{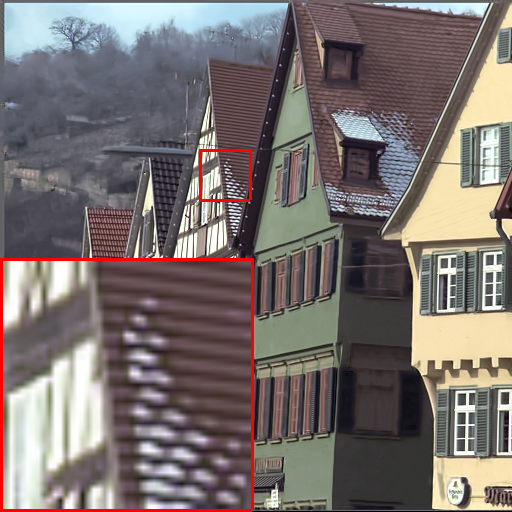}}
		\caption{Denoising results on image kodim08 in Kodak PhotoCD Dataset by different methods (noise level $\sigma_r=30,\sigma_g=10,\sigma_b=50$).(a) Ground Truth. (b) Noisy Image:17.46dB. (c) CBM3D,PSNR=26.83dB. (d) NCSR,PSNR=25.07dB. (e) EPLL,PSNR=26.02dB. (f) DnCNN,PSNR=28.32dB. (g) FFDNet,PSNR=28.18dB. (h) MCWNNM,PSNR=26.98dB. (i) MBWPNM,PSNR=28.83dB. The images are better to be zoomed in on screen.}
		\label{fig3}
	\end{center}
	\vskip -0.2in
\end{figure*}

We evaluate the completing methods on 24 color images from Kodak PhotoCD Dataset\footnote{http://r0k.us.graphics/kodak/}. Zero mean additive white Gaussian noise with four noise levels are added to those test images to generate the degraded observations. The averaged PSNR results are shown in Table~\ref{table1} and the best result is highlighted in bold. According to the discussion of power $p$, we choose $p= \{0.99, 0.98, 0.96, 0.95\}$ for each test in Table~\ref{table1}, respectively. On average, we can see that MBWPNM achieves the highest PSNR values on 2 of 4 noise levels. Compared with MCWNNM, our method achieves $0.35$dB, $0.87$dB, $0.9$dB and $0.16$dB improvements, respectively.

In Figure~\ref{fig3}, we compare the visual quality of the denoised image by the competing algorithms. One can see that methods CBM3D, DnCNN and FFDNet over-smooth in the roof area of image kodim08, while NCSR, EPLL and MCWNNM remain noise or generate more color artifacts. MBWPNM recovers the image well in visual quality than other methods. More visual comparisons can be found in the supplementary material.
\subsubsection{Real Color Images Denoising}
Different from AWGN, the real-world noise is signal dependent, and cannot be modeled by an explicit distribution. To demonstrate the robustness of our method, we compare MBWPNM with the competing methods\cite{Dong2013,Zhang2018,xu2017multi,dabov2007color,zoran2011learning,xu2018external} on three representative datasets.

The first dataset is provided in\cite{nam2016holistic}, which includes 11 indoor scenes, 500 JPEG images per scene. The ground truth noise-free images are generated by computing the mean image of each scene, and author of \cite{nam2016holistic} cropped 15 smaller images of size $512\times512$ for experiments because of original image with resolution $7360\times4912$. These images are shown in supplementary material. Quantitative comparisons, including PSNR result and averaged computational time(ACT), are listed in Table~\ref{table2}. The best results are highlighted in bold. MBWPNM achieves the highest averaged PSNR in all competing methods. It achieves $0.04$dB improvement over MCWNNM method in average and outperforms the benchmark CBM3D method by $1.03$dB in average. One can see that FFDNet based on convolutional neural nets (CNNs) has no longer performed as well as simulated experiments. This point can be confirmed in the following tests. More visual comparisons can be found in the supplementary material.

The other two datasets are provided in \cite{xu2018real} and \cite{abdelhamed2018high}, which both are much more comprehensive than the first one \cite{xu2018real}. The detailed description of these datasets is listed in supplementary file. The PSNR and ACT results of the competing algorithms are reported in Table~\ref{table3}. We can see again that MBWPNM achieves much better performance than the other competing methods. Because of limited space, more visual comparisons can be found in the supplementary file.

\textbf{Comparison on speed}. We compare the average computational time (second) of different methods, which is shown in Table 2 and 4. All experiments are implemented in Matlab on a PC with 3.2GHz CPU and 8GB RAM. The fastest result is highlighted in bold. We can see that FFDNet is the fastest on three datasets and it needs about 4 seconds, while the MBWPNM generally costs about one fifth time compared with MCWNNM. Noted that CBM3D and FFDNet are implemented with compiled C++ mex-function, while NCSR, EPLL, Guided, MCWNNM and MBWPNM are implemented purely in Matlab.

\begin{table*}[ht]
	\caption{PSNR (dB) and ACT (s) of different methods on datasets.}
	\label{table3}
	\vskip 0.15in
	\begin{center}
		\begin{small}
			\begin{sc}
				\begin{tabular}{ccccccccc}
					\toprule
							  & Index & CBM3D & NCSR & EPLL & Guided & FFDNet & MCWNNM & MBWPNM \\
					\midrule
					\multirow{2}*{Dataset 2} & PSNR  & 38.11 & 36.49 & 35.97 & 38.37 & 35.98 & 38.57 & \textbf{38.59}\\
							  & Time  & 6     & 1174  & 427   & 52    & \textbf{4}    & 171 & 38\\
				    \midrule
				    Dataset 3 & PSNR  & 37.20 & 30.21 & 27.48 & 32.55 & 27.69 & 37.82 & \textbf{37.83}\\
				              & Time  & 8     & 1284  & 426   & 48    & \textbf{4}    & 527 & 99\\
					\bottomrule
				\end{tabular}
			\end{sc}
		\end{small}
	\end{center}
	\vskip -0.1in
\end{table*}

\subsection{MSI Denoising}
Different from color image denoising, in order to give an overall evaluation for the spatial and spectral quality, four quantitative picture quality indices (PQI) are employed: PSNR, SSIM\cite{wang2004image}, ERGAS\cite{wald2002data} and SAM\cite{yuhas1993determination}. The larger PSNR and SSIM are, and the smaller ERGAS and SAM are, the better the recovered MSIs are.

The Columbia Multispectral Database (CAVE)\cite{yasuma2010generalized}is utilized in our simulated experiment. The noisy MSIs are generated by adding AWGN with different variance to each of the bands. Specifically, for the number of noise variances, we set them starting from one with the tolerance of 1, 2, 3 and 4.

We choose $p= \{1, 0.95, 0.9, 0.8\}$ for each test. For each noise setting, all of the four PQI are shown in supplementary material. On averaged PSNR, MBWPNM achieves $0.13$dB, $1.46$dB, $1.08$dB, $0.51$dB improvements over LLRT, while achieves the highest PSNR values in all competing methods. Generally, in four comparative tests, MBWPNM achieves the best performance on 11 of 16 quantitative assessments.

\section{Conclusion}
In this paper, we propose a multi-band weighted $l_p$ norm minimization (MBWPNM) model for color image and MSI denoising, which preserves the power of MCWNNM and is flexible to providing different rank components with different treatments in the practical applications. To solve the MBWPNM model, its equivalent form is deduced and it can be efficiently solved via a generalized soft-thresholding (GST) algorithm. We also proved that, when the weights are in a non-descending order, the solution obtained by the GST algorithm is still global optimum. We then applied the proposed MBWPNM algorithm to color image and MSI denoising. The experimental results on synthetic and real datasets demonstrat that the MBWPNM model achieves significant performance gains over several state-of-art methods, such as CBM3D and BM4D. In the future, we will extend the MBWPNM to other applications in computer vision problems.


\bibliography{example_paper}
\bibliographystyle{icml2019}

\end{document}